\pdfoutput=1
\documentclass[11pt]{article} 
\usepackage{times}
\usepackage[letterpaper, left=1in, right=1in, top=1in, bottom=1in]{geometry}

\usepackage[utf8]{inputenc} 
\usepackage[T1]{fontenc}    
\usepackage{url}            
\usepackage{booktabs}       
\usepackage{amsfonts}       
\usepackage{nicefrac}       
\usepackage{microtype}      
\usepackage{mkolar_definitions}
\usepackage{xspace}
\usepackage{amsmath}
\usepackage{algorithm}
\usepackage{algorithmicx,algpseudocode}
\usepackage{color}
\usepackage{enumitem}
\usepackage{comment}
\usepackage{bm}
\usepackage{graphicx}

\usepackage{apptools}
\usepackage[page, header]{appendix}
\usepackage{titletoc}

\usepackage[scaled=.9]{helvet}

\usepackage{url}
\usepackage{authblk}
\usepackage{caption}

\usepackage[dvipsnames]{xcolor}
\usepackage[colorlinks=true, linkcolor=blue, citecolor=blue]{hyperref}

\usepackage{tikz,pgfplots}
\usetikzlibrary{decorations.pathreplacing}
\usepgfplotslibrary{fillbetween}
\usepackage{wrapfig}

\newcommand{\poly}{\mathrm{poly}}

\newcommand{\reg}{\mathrm{Reg}}
\newcommand{\gap}{\mathrm{gap}}
\newcommand{\clip}{\mathrm{clip}}
\newcommand{\relevant}{\mathrm{rel}}
\newcommand{\dom}{\mathrm{dom}}
\newcommand{\diam}{\mathrm{diam}}

\newcommand{\version}{arxiv}

\usepackage{natbib}
\bibpunct{(}{)}{;}{a}{,}{,}

\newcount\Comments  
\definecolor{darkgreen}{rgb}{0,0.5,0}
\definecolor{darkred}{rgb}{0.7,0,0}
\definecolor{teal}{rgb}{0.3,0.8,0.8}
\definecolor{orange}{rgb}{1.0,0.5,0.0}
\definecolor{purple}{rgb}{0.8,0.0,0.8}
\newcommand{\kibitz}[2]{\ifnum\Comments=1{\textcolor{#1}{\textsf{\footnotesize #2}}}\fi}

\usepackage{authblk}

\title{Provably adaptive reinforcement learning in metric spaces}
\date{}
\author[1]{Tongyi Cao\thanks{tcao@cs.umass.edu}}
\author[2]{Akshay Krishnamurthy\thanks{akshaykr@microsoft.com}}
\affil[1]{University of Massachusetts, Amherst, MA}
\affil[2]{Microsoft Research, New York, NY}

\begin{document}

\maketitle

\vspace{-1cm}
\begin{abstract}
We study reinforcement learning in continuous state and action spaces
endowed with a metric. We provide a refined analysis of a variant of the algorithm
of Sinclair, Banerjee, and Yu (2019) and show that its regret scales
with the \emph{zooming dimension} of the instance. This parameter,
which originates in the bandit literature, captures the size of the
subsets of near optimal actions and is always smaller than the
covering dimension used in previous analyses. As such, our results are
the first provably adaptive guarantees for reinforcement learning in
metric spaces.

\end{abstract}

\section{Introduction}
\label{sec:introduction}

In reinforcement learning (RL), an agent learns to select actions to
navigate a state space and accumulates reward.  In terms of theoretical
results, the majority of results address the tabular setting, where
the number of states and actions are finite and comparatively
small. However, tabular problems are rarely encountered in practical
applications, as state and action spaces are often large and may even
be continuous. To address these practically relevant settings, a
growing body of work has developed algorithmic principles and
guarantees for reinforcement learning in continuous spaces.

In this paper, we contribute to this line of work on reinforcement
learning in continuous spaces. We consider episodic RL where the joint
state-action space is endowed with a metric and we posit that the
optimal $Q^\star$ function is \emph{Lipschitz continuous} with respect
to this metric. This setup has been studied in several recent works
establishing worst case regret bounds that scale with the covering
dimension of the metric
space~\citep{song2019efficient,sinclair2019adaptive,touati2020zooming}. While
these results are encouraging, the guarantees are overly pessimistic,
and intuition from the special case of Lipschitz bandits suggests that
much more adaptive guarantees are achievable. In particular, while the
Lipschitz contextual bandits setting of~\citet{slivkins2014contextual}
is a special case of this setup, no existing analysis recovers his
adaptive guarantee that scales with the \emph{zooming dimension} of
the problem.

\paragraph{Our contribution.}
We give the first analysis for reinforcement learning in metric spaces
that scales with the zooming dimension of the instance instead of the
covering dimension of the metric space. The zooming dimension,
originally defined by~\citet{kleinberg2019bandits} in the context of
Lipschitz bandits, measures the size of the set of near-optimal
actions and can be much smaller than the covering dimension in
favorable instances. For reinforcement learning, the natural
generalization is to measure near-optimality relative to the $Q^\star$
function; this recovers the definition of~\citet{kleinberg2019bandits}
and~\citet{slivkins2014contextual} for bandits and contextual bandits,
respectively as special cases. As a consequence, our guarantees also
strictly generalize theirs to the multi-step reinforcement learning
setting. In addition, our guarantee addresses an open problem
of~\citet{sinclair2019adaptive} by characterizing problems where
refined guarantees are possible.


Our result is based on a refined analysis of a variant of the algorithm
of~\citet{sinclair2019adaptive}. This algorithm uses optimism to select
actions and an adaptive discretization scheme to carefully refine a
coarse partition of the state-action space to focus (``zoom in'') on
promising regions. Adaptive discretization is essential for obtaining
instance-dependent guarantees, but the bounds
in~\citet{sinclair2019adaptive} do not reflect this favorable behavior.

At a technical level, the main challenge is that, unlike in bandits,
we cannot upper bound the number of times a highly suboptimal arm will
be selected by the optimistic strategy. Analysis for the bandit
setting uses these upper bounds to prove that the adaptive
discretization scheme will not zoom in on suboptimal regions, which is
crucial for the instance-dependent bounds. However, in RL, the
algorithm actually can zoom in on and select actions in suboptimal
regions, but only when there is significant error at later time
steps. Thus, in the analysis, we credit error incurred from a highly
suboptimal region to the later time steps, so we can proceed as if we
never zoomed in on this region at all. Formally, this analysis uses
the \emph{clipped regret decomposition} of~\citet{simchowitz2019non}
as well as a careful bookkeeping argument to obtain the
instance-dependent bound.

\paragraph{Changes from the initial version.} 
The present version of the paper corrects an error in the version
published in NeurIPS 2020. The differences are both in the algorithm,
which is no longer identical to that of~\citet{sinclair2019adaptive},
and in the analysis, which is somewhat more involved. The changes
address an issue that arises when a child ball inherits updates from
its parent, which results in each sample appearing many times with the
same weight $\alpha_t^i$ in the recursive regret decomposition used in the tabular analysis of~\citet{jin2018q}, displayed in~\pref{eq:old_q_bd}. This
ultimately compromises the final regret bound, which crucially uses
that these weights form a convergent series.

The fix is that child balls no longer inherit data from the parent so
that every interaction tuple (of state, action, reward, next state)
results in exactly one update. This ensures that the $\alpha_t^i$ weight sequences
converge, but is also problematic, as child balls are initialized with
large bonuses, so the confidence sum does not capture the zooming
property we hope to exhibit. We resolve this latter issue with a
\emph{buffering phase} where a child ball is slowly updated but is never
played. Specifically, once a parent ball has received enough samples,
we split it and mark the children as buffering. While the children are
buffering, we continue to use the parent for action selection and
mostly continue to update the parent, but every $H+1^{\textrm{st}}$
update is instead performed on the child. Once the child has enough
updates that the bonus is small, we move it out of the
buffering phase and can safely use it for decision making.

Unfortunately, the buffering approach means that the parent ball is
periodically chosen but not updated, which again results in a weight
sequence where some terms (specifically every $H^{\textrm{th}}$ term)
appears twice. However this sequence is much more benign than the one
that arises if we re-use samples. Indeed, we can show that this new
sequence is convergent via a new amortizing argument that relates it
to the original one in~\citet{jin2018q}.

The final challenge is that now the parent ball remains active for
much longer. This results in a final regret bound that now scales
polynomially with $\Lambda$, the maximum number of children that a
parent can have (or the doubling constant of the metric space), and
additionally is polynomially worse in its dependence on the horizon
$H$ than the bound claimed in the NeurIPS 2020 version of the paper. 
On the
other hand, the new bound still captures the adaptive and zooming
nature of the algorithm.



\section{Preliminaries}

We consider a finite-horizon episodic reinforcement learning setting
in which an agent interacts with an MDP, defined by a tuple
$(\Scal,\Acal,H,\PP,r)$. Here
$\Scal$ the state space, $\Acal$ is the action space, $H \in \NN$ is
the horizon, $\PP$ is the transition operator and $r$ is the reward
function. Formally, $\PP: \Scal\times\Acal \to \Delta(\Scal)$ and
$r:\Scal\times\Acal \to [0,1]$ where $\Delta(\cdot)$ denotes the set
of distributions over its argument.\footnote{Deterministic rewards
  simplifies the presentation but has no bearing on the final
  results. In particular, we can handle stochastic bounded rewards
  with minimal modification to the proofs.}

A (nonstationary) policy $\pi$ is a mapping from states to
distributions over actions for each time. Every policy has
non-stationary value and action-value functions, defined as
\begin{align*}
V_h^\pi(x)\defeq \EE_\pi\sbr{\sum_{h'=h}^H r_{h'}(x_{h'},a_{h'}) \mid x_h = x}, \qquad Q_h^\pi(x,a) \defeq r_h(x,a) + \EE\sbr{V_{h+1}^\pi(x') \mid x,a}.
\end{align*}
Here $\EE_\pi\sbr{\cdot}$ denotes that all actions are chosen by
policy $\pi$ and transitions are given by $\PP$. The optimal policy
$\pi^\star$ and optimal action-value function $Q^\star$ are defined
recursively as
\begin{align*}
Q_h^\star(x,a) \defeq r_h(x,a) + \EE\sbr{\max_{a'}Q^\star(x',a')\mid x,a}, \qquad \pi_h^\star(x) = \argmax_{a} Q_h^\star(x,a).
\end{align*}
The optimal value function $V_h^\star$ is defined analogously.

The agent interacts with the MDP for $K$ episodes, where in episode $k$
the agent picks a policy $\pi_k$ and we generate the trajectory $\tau_k =
(x^k_1,a^k_1,r^k_1,x^k_2,a^k_2,r^k_2\ldots,x^k_H,a^k_H,r^k_H)$ where (1)
$x^k_1$ is chosen adversarially, (2) $a_h^k=\pi_k(x_h^k)$, (3)
$x^k_{h+1} \sim \PP(\cdot \mid x^k_h,a^k_h)$, (4) $r^k_h =
r(x^k_h,a^k_h)$.  We would like to choose actions to maximize the
cumulative rewards $\sum_{h=1}^H r^k_h$.

Equipped with these definitions, we can state our performance
criterion. Over the course of $K$ episodes, we would like to
accumulate reward that is comparable to the optimal policy, formalized
via the notion of regret:
\begin{align*}
\reg(K) \defeq \sum_{k=1}^K \rbr{V_1^\star(x_1^k) - \sum_{h=1}^H r_h^k}.
\end{align*}
In particular, we seek algorithms with regret rate that is sublinear
in $K$. Note that we have not assumed that $|\Scal|$ and $|\Acal|$ are
finite, and we also allow for the starting state $x_1^k$ to be chosen
adversarially in each episode.



\subsection{Metric spaces.}

Instead of assuming that $|\Scal|$ and $|\Acal|$ are finite, we will
posit a metric structure on these spaces. We recall the key
definitions for metric spaces. A space $Y$ equipped with a function
$\Dcal:Y \times Y \to \RR_+$ is a \emph{metric space} if $\Dcal$
satisfies (a) $\Dcal(y,y') = 0$ iff $y = y'$ (b) $\Dcal$ is symmetric,
and (c) $\Dcal$ satisfies the triangle inequality $\Dcal(x,y) \leq
\Dcal(x,z) + \Dcal(z,y)$. If these properties hold then $\Dcal$ is
called a \emph{metric}.  For a radius $r>0$, we use the notation
$B(y,r)\defeq\{y' \in Y: \Dcal(y,y') < r\}$ to denote the open ball
centered at $y$ with radius $r$. For a subset $Y'\subseteq Y$ the
\emph{diameter} is defined as $\diam(Y') \defeq \sup_{y,y' \in
  Y'}\Dcal(y,y')$.
We also use the standard notions of covering and packing to measure
the size of metric spaces.

\begin{definition}[Notions of size]
  We define the following notions of size for a metric space. 
  \begin{itemize}
    \item A \emph{covering} of $Y$ at scale $r$ (also called an $r$-covering) is
a collection of subsets of $Y$, each with diameter at most $r$, whose
union equals $Y$. The minimum number of subsets that form an
$r$-covering is the $r$-covering number, denoted $N_r(Y)$. 
\item  A \emph{packing} of $Y$ at scale $r$ (also called an $r$-packing) is a
collection of points $Z\subset Y$ such that $\min_{z\ne z' \in Z}
D(z,z') \geq r$. The maximum number of points that form an $r$-packing
is the $r$-packing number, denoted $N_r^{\text{pack}}(Y)$.  
\item An \emph{$r$-net} of $Y$ is an $r$-packing $S\subset Y$ for which
  $\{B(y,r)\}_{y \in S}$ covers $Y$.
\item Define the \emph{doubling constant} $\Lambda(Y) := \max_{r>0, y \in Y} N_{r/2}(B(y,r))$, which is the maximum number of balls of radius $\nicefrac{r}{2}$ required to cover some ball of radius $r$.
\end{itemize}
\end{definition}

These definitions also apply to subsets of the metric space, which
will be important for our development. Also note that
$N_{2r}^{\text{pack}}(Y) \leq N_r(Y) \leq N_r^{\text{pack}}(Y)$.




\subsection{Main Assumptions}
We now state the main assumptions that we adopt in our analysis. These
or closely related assumptions are standard in the literature on
bandits and reinforcement learning in metric
spaces~\citep{song2019efficient,sinclair2019adaptive,touati2020zooming,slivkins2014contextual}.
\begin{assum}
\label{assum:diameter}
$(\Scal \times \Acal, \Dcal)$ is a metric space with finite diameter $\diam(\Scal \times \Acal) = d_{\max} < \infty$.
\end{assum}

\begin{assum}
\label{assum:lipschitz}
For every $h\in [H]$, $Q^{\star}_h$ is $L$-Lipschitz continuous with respect to $\Dcal$:
\begin{align}
\forall (x,a), (x',a'): \abr{ Q^\star_h(x,a) - Q^\star_h(x',a')} \le L\cdot\Dcal((x,a), (x',a')).\label{eq:q_lip}
\end{align}
Additionally $V_h^\star$ is $L$-Lipschitz with respect to the metric $\Dcal_\Scal: (x,x') \mapsto \min_{a,a'}\Dcal((x,a),(x',a'))$:
\begin{align}
\forall x,x': \abr{ V^\star_h(x) - V^\star_h(x')} \le L\cdot\min_{a,a'}\Dcal((x,a), (x',a')).\label{eq:v_lip}
\end{align}
\end{assum}

Assumption~\ref{assum:diameter} is a basic regularity condition, while
the first part of Assumption~\ref{assum:lipschitz} imposes continuity
of the $Q^\star$ function. In particular, Lipschitz-continuity
characterizes how the metric structure influences the reinforcement
learning problem. These assumptions appear in prior work,
and we note that~\eqref{eq:q_lip} is strictly weaker than assuming
that $\PP$ is Lipschitz
continuous~\citep{kakade2003exploration,ortner2012online}.

The second part of Assumption~\ref{assum:lipschitz} reflects an
additional structural assumption on the problem, which is a departure
from previous work. In detail,~\eqref{eq:v_lip} posits that the
optimal value function $V_h^\star$ is $L$-Lipschitz with respect to a
metric defined only on the states that is derived from the original
one. This metric is dominated by the original one since for each
$(x,x',a)$ we have $\min_{a_1,a_2}\Dcal((x,a_1),(x',a_2)) \leq
\Dcal((x,a),(x',a))$, so this assumption is not directly implied
by~\eqref{eq:q_lip}. However, whenever $\Dcal$ is
sub-additive in the sense that $\Dcal((x,a),(x',a')) \leq
\Dcal_\Scal(x,x') + \Dcal_\Acal(a,a')$, then the assumption holds
trivially. Sub-additivity holds for most metrics of interest,
including those induced by $\ell_p$ norms for $p \geq 1$. As such,
we do not view this assumption as particularly restrictive.

\subsection{Related work}
Reinforcement learning in the tabular setting, where the state and
action spaces are finite, is relatively
well-understood~\citep{azar2017minimax,dann2017unifying,zanette2019tighter}. Of
this line of work, the two most related papers are those of
of~\citet{jin2018q} and~\citet{simchowitz2019non}. Our results build
on the model-free/martingale analysis of~\citet{jin2018q}, which has
been used in recent work on RL in metric
spaces~\citep{song2019efficient,sinclair2019adaptive,touati2020zooming}. We
also employ techniques from the gap-dependent analysis
of~\citet{simchowitz2019non}. In particular, we use a version of their
``clipping'' argument, as we will explain in
Section~\ref{sec:analysis}.

Moving beyond the tabular setting, several papers study reinforcement
learning in metric spaces, originating with the results
of~\citet{kakade2003exploration}
(c.f.,~\citet{ortner2012online,ortner2013adaptive,song2019efficient,yang2019learning,sinclair2019adaptive,touati2020zooming}). Of
these, the most related result is that of~\citet{sinclair2019adaptive}
who study the adaptive discretization algorithm and give a worst-case
regret analysis, showing that the algorithm has a regret rate of
$K^{\frac{d+1}{d+2}}$ where $d$ is the covering dimension of the
metric space. Essentially the same results appear
in~\citet{touati2020zooming}, although the algorithm is slightly
different. However, none of these results give sharper
instance-dependence guarantees that reflect benign problem structure,
as we will obtain.

For the special case of (contextual) bandits, several
instance-dependent guarantees that yield improved regret rates
exist~\citep{auer2007improved,valko2013stochastic,kleinberg2019bandits,bubeck2011x,slivkins2014contextual,krishnamurthy2019contextual}. For
non-contextual bandits, the results and assumptions vary considerably,
but most results quantify a benign instance in terms of the size of
the set of near-optimal actions. The formulation that we adopt is the
notion of \emph{zooming dimension}, which measures the growth rate of
the $r$-packing number of the set of $O(r)$-suboptimal arms. This
notion has been used in several works on bandits and contextual
bandits in metric spaces, and we will recover some of these results as
a special case of our main theorem.

\section{Main Results}

Our main result is a regret bound that scales with the \emph{zooming
  dimension}. We introduce this parameter with a sequence of
definitions. First, we define the $\gap$ function, which describes the
sub-optimality of an action $a$ for state $x$.
\begin{definition}[Gap]
For any $(x,a) \in \Scal \times \Acal$, for $h\in [H]$, the stage-dependent sub-optimality gap is
\begin{align*}
 \gap_h(x,a) \defeq V^\star_h(x) - Q_h^{\star}(x,a).
 \end{align*}
\end{definition}

We use the gaps to define the subset of the metric space that is near-optimal.
\begin{definition}[Near-optimal set]
\label{def:good_set}
We define near-optimal set as
\begin{align*}
\Pcal_{h,r}^{Q^\star} \defeq \cbr{(x, a) \in \Scal \times \Acal: \gap_h(x,a) \leq \rbr{\frac{2(H+1)}{d_{\max}}+ 2L} r}.
\end{align*}
\end{definition}
Intuitively, $\Pcal_{h,r}^{Q^\star}$ is the set of state-action pairs
with gap that is $O(r)$ at stage $h$. The constant in the definition
is a consequence of our analysis, but it is quite similar to the
constant in the definition of~\citet{slivkins2014contextual} for
contextual bandits. In particular, he considers $d_{\max}=1, H=1,L=1$
and obtains a constant of $12$, while we obtain a constant of $6$ in
this case.


Finally, we define the zooming number and the zooming dimension.
\begin{definition}[Zooming number and dimension]
\label{def:zooming}
The $r$-zooming number is the $r$-packing number of the near-optimal
set $\Pcal_{h,r}^{Q^\star}$, that is
$N_{r}^{\text{pack}}(\Pcal_{h,r}^{Q^\star})$. The stage-dependent
zooming dimension is defined as
\begin{align*}
z_{h,c} \defeq \inf\cbr{ d>0 :  N^{\text{pack}}_{r}(\Pcal_{h,r}^{Q^\star}) \le cr^{-d}, \forall r \in (0,d_{\max}]}.
\end{align*}
The zooming dimension for the instance as the largest among all stages $z_c = \max_{h\in[H]} z_{h,c}$.
\end{definition}

\begin{wrapfigure}{R}{0.4\textwidth}
\ifthenelse{\equal{\version}{arxiv}}{\vspace{-1.5cm}}{}
\begin{center}
\definecolor{b1}{rgb}{0,0,1.0}
\definecolor{b2}{rgb}{0.2,0.2,1.0}
\definecolor{b3}{rgb}{0.4,0.4,1.0}
\begin{tikzpicture}
\draw[thick] (0,0) -- (4,0) node[anchor=north west] {$x$};
\draw[thick] (0,0) -- (0,3) node[anchor=south east] {$a$};


\draw[name path=C2,line width=0.1mm, b3] (0,0.25) .. controls (1,0.0) and (2,2.75) .. (4,1.75);
\draw[name path=B2,line width=0.1mm, b2] (0,0.5) .. controls (1,0.25) and (2,3) .. (4,2);
\draw[name path=A2,line width=0.1mm, b1] (0,0.7) .. controls (1,0.45) and (2,3.2) .. (4,2.2);
\draw[name path=A1,line width=0.1mm, b1] (0,0.8) .. controls (1,0.55) and (2,3.3) .. (4,2.3);
\draw[name path=B1,line width=0.1mm, b2] (0,1) .. controls (1,0.75) and (2,3.5) .. (4,2.5);
\draw[name path=C1,line width=0.1mm, b3] (0,1.25) .. controls (1,1) and (2,3.75) .. (4,2.75);
\tikzfillbetween[of=A1 and A2] {b1,opacity=0.8};
\tikzfillbetween[of=A1 and B1] {b2,opacity=0.6};
\tikzfillbetween[of=A2 and B2] {b2,opacity=0.6};
\tikzfillbetween[of=B1 and C1] {b3,opacity=0.4};
\tikzfillbetween[of=B2 and C2] {b3,opacity=0.4};
\draw[thick,|-|] (1.5, 1.91) -- (1.5, 1.41);
\node [align=left] at (2.8,1.25) {near optimal \\ actions for $x$};
\end{tikzpicture}
\end{center}
\ifthenelse{\equal{\version}{arxiv}}{\vspace{-0.75cm}}{}
\caption{An example where the zooming dimension is $1$ while the
  the covering dimension is $2$. }
\label{fig:zooming_example}
\end{wrapfigure}
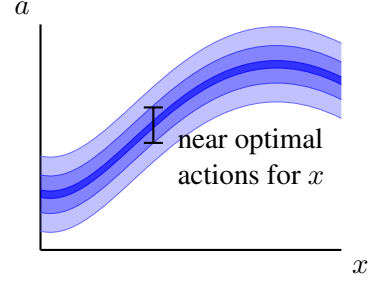
Intuitively, the zooming dimension measures how the near-optimal region
grows as we change the sub-optimality level $r$. Importantly, we use
$r$ both to parametrize the radius in the packing number and the
sub-optimality. Thus, the zooming number captures how many $r$-separated points
can be packed into the $O(r)$ sub-optimal
region.

The more standard notion of complexity of a metric space is the
\emph{covering dimension}, defined as
\begin{align*}
d_c \defeq \inf\{d > 0, N_r^{\text{pack}}( \Scal \times \Acal) \leq cr^{-d}, \forall r \in (0,d_{\max}]\}.
\end{align*}

Examining the definitions, it is clear that we have $z_c \leq d_c$,
since the packing numbers are only smaller. However, in benign
instances where the sub-optimal region concentrates to a low
dimensional manifold, we may have $z_c < d_c$ (and possibly much
smaller), which will enable sharper regret bounds. An example is
illustrated in Figure~\ref{fig:zooming_example}, where the set of
near-optimal actions concentrates on a narrow band for each $x$. Thus
the entire space and hence the covering dimension is $2$-dimensional,
but the zooming dimension is $1$. More generally, if $\Scal$ is a
$d_S$ dimensional space and $\Acal$ is a $d_A$ dimensional space, then
the covering dimension could be $\Omega(d_S + d_A)$ while the zooming
dimension could be as small as $O(d_S)$.

With these definitions, we can now state the main theorem.
\begin{theorem}
\label{thm:main}
For any initial states $\{x_1^k : k\in [K]\}$, and any $\delta \in
(0,1)$, with probability at least, $1-\delta$ Adaptive Q-learning has
the following regret\footnote{Throughout the paper $\tilde{O}(\cdot)$ suppresses logarithmic dependence in its argument.}
\begin{align*}
\reg(K) \le & \tilde{O} \left((H^{3/2}+\sqrt{H\Lambda}) \inf_{r_0 \in (0, d_{\max}]} \left( \sum_{h=1}^{H}\sum_{r=d_{\max}2^{-i},r\ge r_0} N^{\text{pack}}_r(\Pcal_{h,r}^{Q^\star}) \frac{  d_{\max}\sqrt{H\Lambda}}{r} + \frac{Kr_0}{d_{\max}} \right)\right)\\
& ~~~~~~~~ +\tilde{O}\rbr{H^2 + \sqrt{H^3K\log(1/\delta)}}.
\end{align*}
\end{theorem}

Before turning to a discussion of the theorem, we state some
corollaries. First, by optimizing $r_0$, we obtain a regret bound in terms of the zooming dimension.
\begin{corollary}
For any initial states $\{x_1^k : k\in [K]\}$, and any $\delta \in
(0,1)$, with probability at least $1-\delta$ Adaptive Q-learning has
$\reg(K) \leq \tilde{O}\rbr{H^{\frac{5}{2}+\frac{1}{2z_c+4}} K^{\frac{z_c+1}{z_c+2}}}$, for any
constant $c>0$.
\end{corollary}

Finally, we recover the regret rate of~\citet{slivkins2014contextual}
in the special case of contextual bandits.
\begin{corollary}[Contextual bandits]
\label{cor:cb}
If $H=1$, then Adaptive Q-learning has regret $\tilde{O}\rbr{K^{\frac{z_c+1}{z_c+2}}}$, which recovers the regret rate of~\citet{slivkins2014contextual}.
\end{corollary}

We now turn to the remarks:
\begin{itemize}
\setlength{\itemsep}{1pt}
\setlength{\parsep}{1pt}
\setlength{\parskip}{1pt}
\item Theorem~\ref{thm:main} gives a regret bound that depends on the
  packing numbers of the near-optimal set
  (Definition~\ref{def:good_set}). This bound should be compared with
  the ``metric-specific'' regret guarantee
  of~\citet{sinclair2019adaptive} or the ``refined regret bound''
  of~\citet{touati2020zooming}. Both of these results have the same
  form as ours, but with
  $N_r^{\text{pack}}(\Scal\times\Acal)$ in the place of
  $N_r^{\text{pack}}(\Pcal_{h,r}^{Q^\star})$. As
  $\Pcal_{h,r}^{Q^\star} \subset \Scal\times\Acal$, our bound
  improves on theirs in this sense, at the cost of a $\sqrt{H\Lambda}$ additional dependence.
\item The more-interpretable bound is in terms of the zooming
  dimension (Definition~\ref{def:zooming}), which highlights the
  dependence on the number of episodes $K$. We obtain a regret rate of
  $K^{\frac{z_c+1}{z_c+2}}$ for any constant $c>0$, which should be
  compared with the non-adaptive rate $K^{\frac{d_c+1}{d_c+2}}$ that
  scales with the covering
  dimension~\citep{song2019efficient,sinclair2019adaptive,touati2020zooming}.\footnote{We always treat $c$ as a universal constant, so its dependence in the regret bounds is suppressed.} As
  the zooming dimension can be smaller than covering dimension (recall
  Figure~\ref{fig:zooming_example}), this bound demonstrates a
  polynomial improvement over non-adaptive approaches.
\item Corollary~\ref{cor:cb} shows that our bound recovers the
  guarantee from~\citet{slivkins2014contextual}, although his bound
  does not require that~\eqref{eq:v_lip} holds. We give a more
  detailed explanation on the necessity of~\eqref{eq:v_lip} in
  Section~\ref{sec:analysis}.  Nevertheless, the fact that we
  essentially recover his bound suggests that our results are the
  natural generalization to multi-step RL.
\item Finally, we remark that we can instantiate the result in the
  tabular setting with finite $\Scal,\Acal$ by taking the metric to be
  $\Dcal((x,a),(x',a')) = \one\{(x,a)\ne (x',a')\}$. In this case we
  obtain a ``partial'' gap-dependent bound of the form:
\begin{align*}
\poly(H)\cdot\rbr{\sqrt{|\Scal|K} + \sum_{h=1}^H\sum_{x \in \Scal} \sum_{a: \gap_h(x,a) > 0} \frac{\log(K)}{\gap_h(x,a)}}.
\end{align*}
This is not a fully gap-dependent bound because of the
$\sqrt{|\Scal|K}$ term, but it does recover an intermediate result
of~\citet{simchowitz2019non}. In particular, this confirms that the
model-free methods can achieve a partial gap-dependent guarantee for
the tabular setting.
\end{itemize}

\ifthenelse{\equal{\version}{arxiv}}{\vspace{-0.5em}}{}
\section{Algorithm}
\ifthenelse{\equal{\version}{arxiv}}{\vspace{-0.5em}}{}

As we have mentioned, the algorithm is based on the Adaptive
$Q$-learning algorithm of~\citet{sinclair2019adaptive}. The pseudocode
is presented in Algorithm~\ref{alg:zooming}. The algorithm adaptively
partitions the state-action space to focus on the informative regions,
and uses optimism to explore the space and drive the agent to regions
with high reward. Compared to~\citet{sinclair2019adaptive}, the main
difference is that when a new partition is formed it does not inherit
the value and sample count from its parent. Instead, it will go
through an additional buffering phase before it is activated and used
for action selection.

\begin{algorithm}[t]
\caption{Adaptive $Q$--learning with zooming dimension}\label{alg:zooming}
\begin{algorithmic}[1]
  \State For $h \in [H]$, initialize $\Fcal_h^1 = \emptyset$, $\Pcal_h^1$ to be a $\frac{d_{\max}}{H}$-net of $\Scal\times\Acal$. 
  \State For $B \in \Pcal_h^1$, define $Q_h^1(B) = H$. $\tilde{n}_h^1(B) = n_h^1(B) = 0$.
  \For{each episode $k = 1,2,\ldots,K$}
  \State Receive $x_1^k$.
  \For{stage $h = 1,2,\ldots,H$}
    \State $B_h^k = \argmax_{B\in \relevant_h^k(x_h^k)} Q_h^k(B)$, $\tilde{n}_h^{k+1}(B_h^k) = \tilde{n}_h^{k}(B_h^k) + 1$. \label{lin:play}
    \State Play action $a_h^k$ for some $(x_h^k, a_h^k) \in \dom_h^k(B_h^k)$, receive $r_h^k, x_{h+1}^k$.
    \If{$\tilde{n}_h^{k+1}(B_h^k) \ge N_{\mathrm{split}}(B_h^k)$}
    \If{$B_h^k$ is not split} split $B_h^k$. \label{lin:split}
    \State Create a set of children $C(B_h^k) = \frac{1}{2}r(B_h^k)\text{-net of } \dom_h^k (B_h^k)$.
    \State Set $\Fcal_h^{k+1} = \Fcal_h^{k} \cup C(B_h^k)$.
    \ElsIf{$\tilde{n}_h^{k+1}(B_h^k) \mod (H+1) = 0$} \label{alg:line_mod}
    \State Find $B' \in C(B_h^k)$ such that $(x_h^k, a_h^k) \in B'$ and set $B_h^k = B'$. \label{lin:switch}
    \EndIf
    \EndIf
    \State Update $n_h^{k+1}(B_h^k) = n_h^{k}(B_h^k) + 1$ and set $t=n_h^{k+1}(B_h^k)$.
    \State $V_{h+1}^k(x_{h+1}^k) = \min \cbr{H, \max_{B \in \relevant_{h+1}^k(x_{h+1}^k)} Q_{h+1}^k(B)}$.
    \State $Q_h^{k+1}(B_h^k) = (1-\alpha_t) Q_h^k(B_h^k) + \alpha_t(r_h^k + b_t + V_{h+1}^k(x_{h+1}^k))$. \label{lin:update}
    \If{$B_h^k \in \Fcal_h^k$ and $n_h^{k+1}(B_h^k) \geq N_{\mathrm{min}}(B_h^k)$ }
    \State Move $B_h^k$ from $\Fcal_h^k$ to $\Pcal_h^k$, i.e. $\Fcal_h^{k+1} = \Fcal_h^k \setminus \{B_h^k\}, \Pcal_h^{k+1}=\Pcal_h^k \cup \{B_h^k\}$.
    \State Set $\tilde{n}_h^{k+1}(B_h^k) = n_h^k(B_h^k)$
    \EndIf
  \EndFor
  \State Advance all other algorithm state (i.e., $\Pcal_h^{k+1} \gets \Pcal_h^k$, etc., if not explicitly updated above)
  \EndFor
\end{algorithmic}
\end{algorithm}

During the execution, the algorithm creates many balls $B \subset
\Scal\times\Acal$ for each stage $h$. For stage $h$ and episode $k$,
we use $\Pcal_h^k$ to denote the set of \emph{active} balls, and
$\Fcal_h^k$ to denote the set of \emph{buffering} balls. When a set of
balls are created, they are first moved to the buffering set, and
balls in this set will not be used for decision making, but may
occasionally be updated.

Each ball $B$ is associated with (1) a radius, denoted $r(B)$, (2) a
domain, denoted $\dom_h^k(B)$, (3) several counters and thresholds
related to the amount of data it has seen, and (4) an optimistic
estimate of $Q_h^\star$. The radius of a ball is $r(B):= \diam(B)$ and 
the domain $\dom_h^k(B)$ is the set of points contained in this ball,
but not in any other active ball with a smaller radius. Formally,
\begin{align*}
\dom_h^k(B) \defeq B \setminus \{\cup_{B' \in \Pcal_h^k: r(B') < r(B)} B'\}.
\end{align*}
For the counters, $n_h^k(B)$ denotes the number of times ball $B$ has
been updated at stage $h$ and episode $k$, while $\tilde{n}_h^k(B)$
denotes the number of times ball $B$ has been ``played'' or used for
decision making. These two counters will not be equivalent in
general. We also use two thresholds: $N_{\mathrm{split}}(B)$ is the
number of updates we must perform before we split $B$ into smaller
balls, and $N_{\mathrm{min}}(B)$ is the number of updates we must
perform before moving $B$ from the buffering set $\Fcal_h^k$ to the
active set $\Pcal_h^k$. 
These
latter two are defined as:
\begin{align*}
N_{\mathrm{split}}(B) := 4N_{\mathrm{min}}(B) := \rbr{\frac{d_{\max}}{r(B)}}^2.
\end{align*}
When a ball is split in line~\ref{lin:split} the resulting balls are
called children and denoted $C(B)$.  Finally, each ball maintains a
scalar $Q_h^k(B)$ which serves as an upper bound on $\max_{(x,a) \in
  B} Q_h^\star(x,a)$.

In stage $h$ of episode $k$, we select the action for state $x_h^k$ as
follows: we consider all the smallest active balls that contains $x_h^k$,
defined as ``relevant'' balls
\begin{align*}
\relevant_h^k(x) \defeq \{B \in\Pcal_h^k \mid \exists a, (x,a) \in \dom_h^k(B)\}.
\end{align*}
Among the relevant balls, the algorithm selects the ball $B_h^k$ with
the highest $Q_h^k(B)$ value and plays an arbitrary action such that
$(x_h^k,a) \in \dom_h^k(B_h^k)$.
We almost always update the ball that we play, except sometimes we invoke line~\ref{lin:switch} where we rebind $B_h^k$ to be one of the children. In this case, we play a certain ball but then update its child.
At the end of the episode, we update the estimated $Q$ value $Q_h^k(B_h^k)$ and increment the sample count $t = n_h^k(B_h^k) + 1$.
The update rule is a form of optimistic $Q$ learning
\begin{align*}
Q_h^{k+1}(B_h^k) &= (1-\alpha_t) Q_h^k(B_h^k) + \alpha_t(r_h^k + b_t + V_{h+1}^k(x_{h+1}^k)),\\
V_{h+1}^k(x) &= \min\cbr{H, \max_{B\in \relevant_{h+1}^k(x)} Q_{h+1}^k(B)}.
\end{align*}
where the $\alpha_t$ is the learning rate and $b(t)$ is the bonus added to ensure that $Q_h^k$ is optimistic. Formally,
\begin{align*}
\alpha_t \defeq \frac{H+1}{H+t}, \qquad b_t \defeq 2\sqrt{\frac{H^3 \log(4HK/\delta)}{t}} + \frac{4Ld_{\max}\sqrt{H\Lambda +\Lambda + 1}}{\sqrt{t}}.
\end{align*}
For all other balls at stage $h$, we set $Q_h^{k+1}(B) \gets
Q_h^k(B)$, with no update.

We split a ball $B$ as soon as $n_h^k(B) \geq N_{\mathrm{split}}(B)$.  When splitting, we create a
set of new ``children'' balls with radius $r(B)/2$ that forms an
$r(B)/2$-net of $\dom_h^k(B)$.
These ``children'' are added to the buffering set $\Fcal_h^k$. 
Once a ball $B \in \Fcal_h^k$ receives $N_{\mathrm{min}}(B)$ updates, we move it to the active set $\Pcal_h^k$ and we can use it for action selection. 
This splitting rule leads to the following
invariant:
\begin{lemma}[Lemma 5.3 in~\cite{sinclair2019adaptive}]
For every $(h,k)\in [H]\times[K]$, we have
\begin{enumerate}
\setlength{\itemsep}{0em}
\setlength{\parsep}{0em}
\setlength{\parskip}{0em}
\item (Covering) The domains of balls in $\Pcal_h^k$ covers $\Scal \times \Acal$.
\item (Separation) For any two balls of radius $r$, their centers are at distance at least $r$.
\end{enumerate}
\end{lemma}


The last component to describe is the warm-starting process for balls in the buffering phase, which is the main difference compared with the algorithm of~\citet{sinclair2019adaptive}.
This process works as follows. A ball $B$ with $n_h^k \geq
N_{\mathrm{split}}(B)$ updates may still be chosen for action
selection if some of its children are still buffering (if no child is
buffering, then, by the definition of $\dom_h^k$, $B$ cannot be
selected). During this phase, every $H+1$ times that we play ball $B$,
we instead use the sample to update one of the children, specifically
the one that contains $(x_h^k,a_h^k)$. In this way, balls in the
buffering set $\Fcal_h^k$ slowly accumulate samples and eventually can
be moved to the active set.

\ifthenelse{\equal{\version}{arxiv}}{\vspace{-0.5em}}{}
\section{Proof sketch}\label{sec:analysis}
\ifthenelse{\equal{\version}{arxiv}}{\vspace{-0.5em}}{}

In this section we describe the main steps of the proof, with details
deferred to the appendix.

It is worth reviewing prior regret analyses for episodic RL~\citep{jin2018q}. The arguments establish a regret
decomposition that relates the estimate $V_1^k$ to $V_1^{\pi_k}$, the
expected reward collected in episode $k$. The decomposition is
recursive in nature, involving differences between $Q_h^k$ and
$Q_h^\star$. These are controlled by the update rule and the design of
the learning rate. In particular, we can bound $Q_h^k - Q_h^\star$ by
an immediate ``surplus'' $\beta_t$ and the downstream value function
error. Formally for any ball $B$ with $(x,a) \in \dom_h^k(B)$
\begin{align}
\label{eq:old_q_bd}
Q_h^k(B) - Q_h^\star(x,a) \leq \one_{[t=0]}H + \sum_{i=1}^t\alpha_t^i(V_{h+1}^{k_i} - V_{h+1}^\star)(x_{h+1}^{k_i}) + \beta_t,
\end{align}
where $t = n_h^k(B), \alpha_i^t = \alpha_i
\prod_{j=i+1}^t(1-\alpha_j)$ and $\beta_t =
2\sum_{i=1}^t\alpha_i^tb_i$. Here $k_i$ is the index of the episode
where $B$ was updated for the $i^{\textrm{th}}$ time. Summing over
all episodes and grouping terms appropriately (and ignoring the buffering process for now), we obtain
\begin{align*}
\sum_{k=1}^K (V_h^k - V_h^{\pi_k})(x_h^k) \leq \sum_{k=1}^K\rbr{H\one_{[n_h^k=0]} + \beta_{n_h^k} + \xi_h^k} + \rbr{1 + \nicefrac{1}{H}}\sum_{k=1}^K\rbr{V_{h+1}^k - V_{h+1}^{\pi_k}}(x_{h+1}^k),
\end{align*}
where $\xi_{h+1}^k$ is a stochastic term that can be ignored for this
discussion. Note that, as long as $V_h^k$ is optimistic (which we will
verify), this also provides a bound on the regret.

For the tabular setting,~\citet{jin2018q} use this regret
decomposition to obtain a worst-case bound. The leading term arises
from the ``surplus'' term $\beta_{n_h^k}$, which leads to a
$\poly(H)\sqrt{SAK}$ regret bound for the tabular setting. On the
other hand for our setting, the splitting rule and the buffering scheme implies that, for any ball $B$, we must have $n_h^k \leq (H\Lambda + \Lambda+1)\rbr{\nicefrac{d_{\max}}{r(B)}}^2$, as we will show.
We can use this to obtain a bound
that depends on the number of active balls at each scale $r$ times
$d_{\max}/r$. If we could bound the number of active balls at scale
$r$ in terms of the packing number
$N_r^{\text{pack}}(\Pcal_{h,r}^{Q^\star})$, then we would obtain the
instance-dependent bound.

Unfortunately, this is not possible. In general, the algorithm will
activate balls outside of the near-optimal region, because we may have
to select a highly suboptimal ball many times to reduce downstream
over-estimation error. So indeed the number of active balls at scale
$r$ could be much larger than the packing of the near-optimal set.


We address this with the following key observation. If the surplus
$\beta_{n_h^k}$ is small compared to gap, and we choose this ball, it
must be the case that the downstream regret is quite large, otherwise
we would not have chosen this ball. If this is true, we can account
for the surplus by adding a small constant fraction of the future
regret. In otherwords, we can ``clip'' the surplus to zero once it is
proportional to the gap, and we only pay a constant factor in the
recursive term.  This is the clipping trick developed
by~\citet{simchowitz2019non} to establish gap dependent bounds for
tabular MDP. Formally instead of~\pref{eq:old_q_bd}, we have the
following lemma.

\begin{lemma}[Clipped upper bound]\label{lem:clipupper}
For any $\delta \in (0,1)$ with probability at least $1-\delta/2$,
$\forall h\in [H]$,
\begin{align*}
Q_h^k(B_h^k) - Q_h^{\star}(x_h^k, a_h^k) &\le \rbr{1+\nicefrac{1}{H}}\rbr{H\one_{[t=0]} + \sum_{i=1}^{t} \alpha_t^i(V_{h+1}^{k_i} - V^\star_{h+1})(x_{h+1}^{k_i})} \\
& ~~~~~~~~ + \clip\sbr{\beta_t \mid \frac{\gap_h(x_h^k, a_h^k)}{H+1}},
\end{align*}
where $t = n_h^k(B), \alpha_i^t = \alpha_i
\prod_{j=i+1}^t(1-\alpha_j)$ and $\beta_t =
2\sum_{i=1}^t\alpha_i^tb_i$ and $\clip[\mu\mid\nu] \defeq \mu\one\{\mu \geq \nu\}$.
\end{lemma}

This bound should be compared with~\eqref{eq:old_q_bd}. On one hand
the recursive term is multiplied by $1+\nicefrac{1}{H}$, but, on the
other, we are able to clip the surpluses $\beta_t$. The former will
exponentiate but will asymptote to $e$, while the latter is crucial
for our instance dependent bounds.

Using this lemma, we can bound the difference between $V_h^k$ and
$V_h^{\pi_k}$.
\begin{lemma}[Clipped recursion, informal]\label{lem:cliprecur_main}
For any $\delta \in (0,1)$, with probability at least $1-\delta/2$, $\forall h\in [H]$,
\begin{align*}
\sum_{k=1}^K (V_h^k - V_h^{\pi^k})(x_h^k) \le & \sum_{k=1}^K \rbr{1+\nicefrac{1}{H}} \rbr{H\one_{[n_h^k(B_h^k) = 0]} + \clip\sbr{\beta_{n_h^k(B_h^k)} \mid \gap_h(x_h^k, a_h^k)/(H+1)} + \xi_{h+1}^k} \\
&+ \rbr{1+\nicefrac{2}{H}}^3 \sum_{k=1}^K (V_{h+1}^k - V_{h+1}^{\pi^k})(x_{h+1}^k),
\end{align*}
where $\xi_{h+1}^k$ is conditionally centered random variable with
range $H$. 
\end{lemma}
We bound $V_1^k - V_1^{\pi_k}$, and by optimism the regret, by
applying Lemma~\ref{lem:cliprecur_main} recursively.

The last step is to show that the sum of clipped surpluses can be
related to the zooming dimension. First note that for any ball $B$, the buffering process implies that it is updated at least 
$\nicefrac{1}{4}\rbr{\nicefrac{d_{\max}}{r(B)}}^2$ times before it
becomes activated. If it becomes activated but only contains points
with large gap, we can always clip the surplus term. Thus all active
balls $B$ that have $r(B) \ll \min_{x,a \in B}\gap(x,a)$ do not
contribute to the regret.

Next, if a ball with radius $r$ contains a point where the gap is
small, we cannot appeal to clipping. However, by Lipschitzness, all
points in the ball must have small gaps, which means that this ball is
contained in the near optimal set at scale $r$. As above, the surplus for
each of these balls contributes at most $\nicefrac{d_{\max}}{r}$ to
the regret. Then, since all balls with radius $r$ are at least $r$
apart and we only incur regret for those entirely contained in the
near-optimal region, we obtain the bound that depends on
$N_r^{\text{pack}}(\Pcal_{h,r}^{Q^\star})$.

\ifthenelse{\equal{\version}{arxiv}}{\vspace{-1em}}{}
\paragraph{Remarks on Assumption~\ref{assum:lipschitz}.}
We give some intuition on why our proof
requires~\eqref{eq:v_lip}, which is slightly stronger than what is
required for the zooming dimension analysis
of~\citet{slivkins2014contextual} for contextual bandits.
In~\citet{slivkins2014contextual}, the optimistic selection rule
ensures that the context-action pairs chosen by the algorithm have
small gap, but this is not true in the multi-step setting. In the RL
setting, we might select an action (in a ball) with a large gap
because the downstream regret is large. In this case, we can clip the
surplus, but we can only clip at the \emph{minimum} gap among all
$(x,a)$ pairs in the ball. To obtain a zooming dimension bound, we
must argue that this ball is contained in the near-optimal set, but
this requires that the value functions, and hence the gaps, are
Lipschitz. We recall that~\eqref{eq:v_lip} is implied
by~\eqref{eq:q_lip} if the metric is sub-additive.

\ifthenelse{\equal{\version}{arxiv}}{\vspace{-0.5em}}{}
\section{Discussion}
\ifthenelse{\equal{\version}{arxiv}}{\vspace{-0.5em}}{}

In this paper, we give a refined analysis of a variant of the Adaptive
Q-learning algorithm of Sinclair, Banerjee and Yu (2019) for sample
efficient reinforcement learning in metric spaces. We show that the
algorithm has a regret bound that depends on the zooming dimension of
the instance, with rate $K^{\frac{z+1}{z+2}}$ when the zooming
dimension is $z$. This improves on the worst-case bound that depends
on the covering dimension, and can be much better when the $Q^\star$
function concentrates quickly onto a low-dimensional set of
actions. The bound also recovers that
of~\citet{slivkins2014contextual} for contextual bandits in metric
spaces, under a slightly stronger assumption.  The key technique is
the clipped regret decomposition of~\citet{simchowitz2019non}, which
we complement with a novel buffering process and a corresponding
book-keeping argument. Our results show that adaptivity to benign
instances is possible in RL with metric spaces, and partially mitigate
the curse of dimensionality in such settings.

\subsection*{Acknowledgements}
We thank Wen Sun and Aleksandrs Slivkins for formative discussions
during the conception of this paper. We also thank Max Simchowitz for
insightful discussions regarding the clipping technique. Finally, we
thank Chicheng Zhang for identifying the error in the previous version
of the paper and for helpful discussions regarding the fix. 

\bibliography{../refs}
\vfill

\clearpage

\appendix
\section{Appendix}

In this section we provide a detailed proof for the main
theorem. First we state some facts about the learning rate and the
algorithm. The first lemma regarding the learning rate sequence is
directly from~\citet{jin2018q}.

\begin{lemma}[Lemma 4.1 from \cite{jin2018q}]\label{lem:lr}
Let $\alpha_t^i := \alpha_i \prod_{j=i+1}^t (1-\alpha_j)$. Then for every $i\geq 1$:
\begin{align*}
\sum_{t=i}^{\infty} \alpha_t^i = 1+\frac{1}{H}.
\end{align*}
\end{lemma}

\begin{figure}[t]
  \begin{center}
    \includegraphics[width=0.5\textwidth]{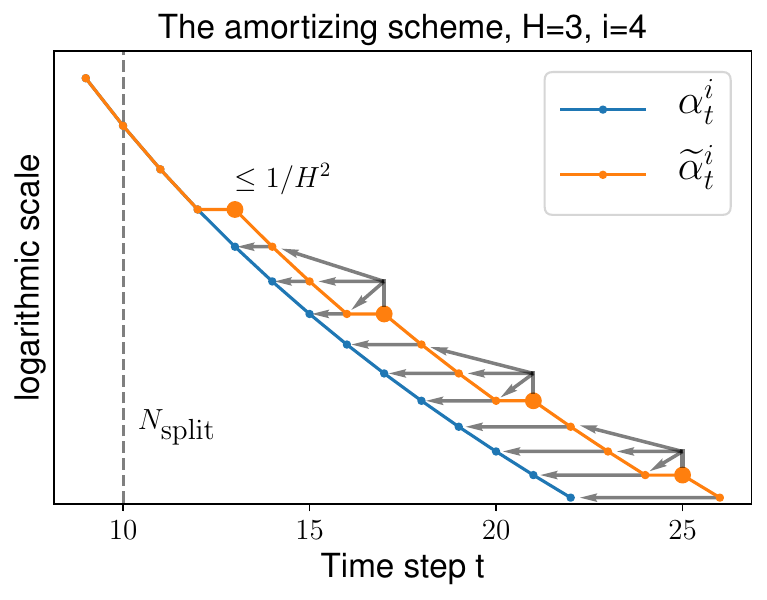}
  \caption{An illustration of the amortizing argument which relates
    the sequence $\tilde{\alpha}_t^i$ to $\alpha_t^i$.}
  \label{fig:amortizing_fig}
  \end{center}
\end{figure}

The next lemma, also regarding the learning rate sequence, is
new. This lemma shows how the skipped updates that arise due to our
buffering scheme do not significantly compromise the regret bound. The
proof is based on an amortizing argument, which is illustrated in
Figure~\ref{fig:amortizing_fig}.
\begin{lemma}\label{lem:lr_new}
Fix $i\geq H^2$ and $T_0\geq 0$. Consider the sequence
$\{\alpha_t^{i}\}_{t\ge i}$, and let $\{\tilde{\alpha}_t^{i}\}_{t\ge
  i}$ be the sequence formed by repeating every
$H^{\text{th}}$ term in $\{\alpha_t^{i}\}_{t\ge i}$ starting at $t =
T_0$. Then
\begin{align*}
\sum_{t=i}^{\infty} \tilde{\alpha}_t^{i} \le \rbr{1+\frac{2}{H}}^2.
\end{align*}
\end{lemma}
\begin{proof}
We can rewrite the sum of $\tilde{\alpha}_t^{i}$ as
\begin{align*}
\sum_{t=i}^{\infty}\tilde{\alpha}_t^{i} = \sum_{t=i}^\infty \alpha_t^{i} + \sum_{j=0}^{\infty} \alpha_{T_0+jH}^{i}.
\end{align*}
We will show how to absorb the second sum into the first, and we will
use a crediting scheme as in Figure~\ref{fig:amortizing_fig}.  The
first observation is that
\begin{align*}
\alpha_{T_0+jH}^{i} \leq \frac{1}{H}\sum_{k=0}^{H-1} \alpha_{T_0 + jH-k}^i,
\end{align*}
since $\alpha_t^i$ is a decreasing sequence. This observation immediately
addresses any term for which the $H$ previous terms appear in the
first sequence. This is any term where $T_0 + jH - i \geq H$.

We just have to handle the terms where $j$ is such
that $T_0+jH-i < H$. In this case we must have $j=0$.
Using the fact that $i \geq H^2$ we obtain
\begin{align*}
\alpha_{T_0}^i \leq \alpha_i^i \leq \frac{H+1}{H+H^2} \leq \frac{1}{H}.
\end{align*}
Putting these two observations together, we haven
\begin{align*}
\sum_{t=i}^{\infty}\tilde{\alpha}_t^{i} &= \sum_{t=i}^\infty \alpha_t^{i} + \sum_{j=0}^{\infty} \alpha_{T_0+jH}^{i}
= \sum_{t=i}^\infty \alpha_t^{i} + \sum_{j=1}^{\infty} \alpha_{T_0+jH}^{i} + \alpha_{T_0}^i\\
& \leq \sum_{t=i}^\infty \alpha_t^i + \frac{1}{H}\sum_{t=i}^\infty \alpha_t^i + \frac{1}{H}
\leq \rbr{1 + \frac{1}{H}}^2 + \frac{1}{H} \leq \rbr{ 1 + \frac{2}{H}}^2.
\end{align*}
The last step uses Lemma~\ref{lem:lr}.
\end{proof}

The next lemma establishes some basic facts on the counters used in the algorithm. 
\begin{lemma}\label{lem:numsample}
For any $k\in [K]$, $h\in [H]$ and ball $B\in \Fcal_h^k$, $B' \in \Pcal_h^k$ we have
\begin{align*}
&n_h^k(B) \in [0, N_{\mathrm{min}}(B) - 1], \quad \tilde{n}_h^k(B) = 0\\
&N_{\mathrm{min}}(B') \le n_h^k(B') \le \tilde{n}_h^k(B') \le (H\Lambda+\Lambda+1) \left(\frac{d_{\max}}{r(B)}\right)^2 =: N_{\mathrm{max}}(B)
\end{align*}
\end{lemma}

\begin{proof}
For $B\in \Fcal_h^k$, the upper bound on the number of updates comes
directly from the rule to move $B$ into the active set. Additionally,
the ball in the buffering set will not be played according to the
definition of $\relevant_h^k(x)$, so $\tilde{n}_h^k(B) = 0$.  For
$B'$, at the time it is added to $\Pcal_h^k$ we have $n_h^k(B') =
\tilde{n}_h^k(B')$. It will only be updated if it is played, but if it
is played it is not necessarily updated, so we have $n_h^k(B') \le
\tilde{n}_h^k(B')$.

The final bound is less obvious. A ball $B$ will no longer be played if
all of its children are in the active set, by definition of
$\relevant_h^k(x)$. Further, by the definition of $\dom_h^k$, when a
ball ``passes down'' its update to a child, that child must be in the
buffering set (otherwise the state action pair is not in the domain of $B$).
Before splitting, $B$ is played at most
$\left(\frac{d_{\max}}{r(B)}\right)^2$ times. After splitting, each
child will be updated at most $\frac{1}{4}\left(\frac{d_{\max}}{r(B) /
  2 }\right)^2 = \left(\frac{d_{\max}}{r(B)}\right)^2$ before being
moved to the active set. Since we have at most $\Lambda$ children (by
definition of the doubling constant) and we play the parent ball
$(H+1)$ times for each update to a child, we obtain the bound
$(H\Lambda+\Lambda+1) \left(\frac{d_{\max}}{r(B)}\right)^2$.
\end{proof}

Next we prove an elementary bound on the bias incurred by some ball. 
\begin{lemma}\label{lem:e6}
For any $(x,a,h,k) \in \Scal \times \Acal \times [H] \times [k]$ and ball $B\in \Pcal_h^k$ with $(x,a) \in \dom_h^k(B)$, if $B$ is updated at step $h$ in episodes $k_1 < k_2 < \dots < k_t < k$, where $t=n_h^k(B)$, then
\begin{align*}
\sum_{i=1}^t \alpha_t^i | Q^\star_h(x_h^{k_i}, a_h^{k_i}) - Q^\star_h(x,a)| \le 4Ld_{\max}\sqrt{(H\Lambda+\Lambda+1)}\frac{1}{\sqrt{t}}.
\end{align*}
\end{lemma}
\begin{proof}
By Lemma~\ref{lem:numsample}, we have $n_h^k(B) \le (H\Lambda+\Lambda+1) \left(\frac{d_{\max}}{r(B)}\right)^2$. Re-arranging, we find that $r(B) \leq d_{\max} \sqrt{\frac{(H\Lambda + \Lambda + 1)}{n_h^k(B)}}$. Of course we always have $\alpha_t^i \leq 1$, and so, by Lipschitzness we have
\begin{align*}
\sum_{i=1}^t \alpha_t^i |Q^\star_h(x_h^{k_i},a_h^{k_i}) - Q^\star_{h}(x,a)| &\le 2Lr(B) \sum_{i=1}^t \alpha_t^i
\le 2Ld_{\max}\sqrt{(H\Lambda+\Lambda+1)}\frac{1}{\sqrt{t}}. \tag*\qedhere
\end{align*}
\end{proof}

To bound the regret, our starting point is an upper bound on the difference between the optimistic $Q$--function and the optimal $Q^\star$ function.
\begin{lemma}\label{lem:e7}
For any $\delta \in (0,1)$ if $\beta_t := 2\sum_{i=1}^t \alpha_t^i b(i)$ then
\begin{align*}
\beta_t \le 8 \sqrt{\frac{H^3 \log(4HK/\delta)}{t}} + 16\frac{Ld_{\max}\sqrt{(H\Lambda+\Lambda+1)}}{\sqrt{t}}.
\end{align*}
Additionally, with probability at least $1-\delta/2$ the following holds simultaneously for all $(x, a, h, k) \in \Scal \times \Acal \times [H] \times [K]$ and ball $B$ such that $(x,a) \in \dom_h^k(B)$: 
\begin{align*}
0\le Q_h^k(B) - Q_h^{\star}(x, a) \le \one_{[t=0]}H + \beta_t + \sum_{i=1}^{t} \alpha_t^i(V_{h+1}^{k_i} - V^\star_{h+1})(x_{h+1}^{k_i}),
\end{align*}
where $t=n_h^k(B)$, and $k_1 < \dots < k_t$ are the episodes where $B$ was previously updated by the algorithm.
\end{lemma}
\begin{proof}
This is a modified version of Lemma E.7 from \cite{sinclair2019adaptive}.
The proof is exactly the same as the original, except that we use a larger bonus term $b(i)$ to account for larger upper bound in Lemma \ref{lem:e6}.
\end{proof}

This bound contains three parts.  The first is an upper bound for the
first step when there is no data.  The second term, $\beta_t$, is the
surplus that we add to ensure optimism.  The third part is an
``average'' of the estimated future regret.  The key observation is
that when $\beta_t$ is small, it can be absorbed into the future
regret.  In this way, we can clip $\beta_t$ proportionally to the
future regret which enables a form of gap dependent regret bound. This
clipping feature is demonstrated in the next lemma. Recall the
definition $\clip[\mu\mid \nu]:=\mu\one\{\mu \geq \nu\}$.
\begin{lemma}[Clipped upper bound]
For any $\delta \in (0,1)$ if $\beta_t := 2\sum_{i=1}^t \alpha_t^i b(i)$.
With probability at least $1-\delta/2$, $\forall h\in [H], k \in [K]$,
\begin{align*}
Q_h^k(B_h^k) - Q_h^{\star}(x_h^k, a_h^k) \le& \rbr{1+\frac{1}{H}}\rbr{\one_{[t=0]}H + \sum_{i=1}^{t} \alpha_t^i(V_{h+1}^{k_i} - V^\star_{h+1})(x_{h+1}^{k_i})} \\
& + \clip\sbr{\beta_t \mid \gap_h(x_h^k, a_h^k) / (H+1)}.
\end{align*}
\end{lemma}
\begin{proof}
We use $a_h^\star:\Xcal \rightarrow \Acal$ to denote a mapping from the state to the optimal action at stage $h$. By the definition of the gap
\begin{align*}
\gap_h(x_h^k, a_h^k) &= Q^\star_h(x_h^k, a_h^\star(x_h^k)) - Q^\star(x_h^k,a_h^k)
\le Q^k_h(B_h^{k\star}) - Q^\star_h(x_h^k,a_h^k) \\
&\le Q^k_h(B_h^{k}) - Q^\star_h(x_h^k,a_h^k)
\le \one_{[t=0]}H + \beta_t + \sum_{i=1}^{t} \alpha_t^i(V_{h+1}^{k_i} - V^\star_{h+1})(x_{h+1}^{k_i}),
\end{align*}
where $B_h^{k\star}$ is any ball in $\relevant_h^k(x_h^k)$ such that
$(x_h^k, a_h^\star(x_h^k))\in \dom_h^k(B_h^{k\star})$ (note that such
a ball must exist).
The first inequality is by the lower bound of
Lemma~\ref{lem:e7}, namely the optimism of $Q_h^k$. 
The
second uses the selection rule of choosing the ball with the
largest $Q_h^k(B)$ among those in $\relevant_h^k(x_h^k)$. The third inequality is by
the upper bound of Lemma~\ref{lem:e7}.

Now we consider two cases, if $\beta_t > \gap_h(x_h^k, a_h^k) /
(H+1)$, the bound is trivially implied by Lemma~\ref{lem:e7}. If
$\beta_t \le \gap_h(x_h^k, a_h^k) / (H+1)$, then
\begin{align*}
\gap_h(x_h^k, a_h^k) &\le \one_{[t=0]}H + \beta_t + \sum_{i=1}^{t} \alpha_t^i(V_{h+1}^{k_i} - V^\star_{h+1})(x_{h+1}^{k_i}) \\
&\le \one_{[t=0]}H + \sum_{i=1}^{t} \alpha_t^i(V_{h+1}^{k_i} - V^\star_{h+1})(x_{h+1}^{k_i}) + \gap_h(x_h^k, a_h^k) / (H+1)
\end{align*}
By re-arranging to move all $\gap$ terms to the left hand side, we have
\begin{align*}
\gap_h(x_h^k, a_h^k) \le \frac{H+1}{H}\rbr{\one_{[t=0]}H + \sum_{i=1}^{t} \alpha_t^i(V_{h+1}^{k_i} - V^\star_{h+1})(x_{h+1}^{k_i})}
\end{align*}
By Lemma~\ref{lem:e7} and our assumption
\begin{align*}
Q_h^k(B_h^k) - Q_h^{\star}(x_h^k, a_h^k) &\le \one_{[t=0]}H + \beta_t + \sum_{i=1}^{t} \alpha_t^i(V_{h+1}^{k_i} - V^\star_{h+1})(x_{h+1}^{k_i})\\
&< \one_{[t=0]}H + \gap_h(x_h^k, a_h^k) / (H+1) + \sum_{i=1}^{t} \alpha_t^i(V_{h+1}^{k_i} - V^\star_{h+1})(x_{h+1}^{k_i})\\
& \le \rbr{1+\frac{1}{H}}\rbr{\one_{[t=0]}H + \sum_{i=1}^{t} \alpha_t^i(V_{h+1}^{k_i} - V^\star_{h+1})(x_{h+1}^{k_i})}. \tag*\qedhere
\end{align*}
\end{proof}

The next step is to replace the future regret to $V^\star$ with the future regret of $V^{\pi_k}$, so that we can solve for the $h=1$ case recursively.

\begin{lemma}[Clipped recursion]\label{lem:cliprecur}
For any $\delta \in (0,1)$ if $\beta_t := 2\sum_{i=1}^t \alpha_t^i b(i)$.
With probability at least $1-\delta/2$, $\forall h\in [H], k \in [K]$,
\begin{align*}
\sum_{k=1}^K (V_h^k - V_h^{\pi^k})(x_h^k) \le& \sum_{k=1}^K\rbr{1+\frac{1}{H}}\rbr{H\one_{[n_h^k(B_h^k)=0]} + \xi_{h+1}^k + \clip\sbr{\beta_{n_h^k(B_h^k)} \mid \frac{\gap_h(x_h^k, a_h^k)}{H+1}}} \\
&+ \rbr{1+\frac{2}{H}}^3 \sum_{k=1}^K(V^k_{h+1} - V_{h+1}^{\pi^k})(x_{h+1}^k),
\end{align*}
where $\xi_{h+1}^k = \EE\sbr{V^\star_{h+1}(x) - V_{h+1}^{\pi_k}(x) \mid x_h^k, a_h^k} - (V_{h+1}^\star - V_{h+1}^{\pi_k})(x_{h+1}^k)$.
\end{lemma}
\begin{proof}
First, consider stage $h$ in episode $k$ and let $B_h^k$ be the ball
that is chosen. Define $t = n_h^k(B_h^k)$. Then by applying the previous lemma, we have
\begin{align*}
V_h^k(x_h^k) - V_h^{\pi^k}(x_h^k) &= \max_{B\in \relevant_h^k(x_h^k)} Q_h^k(B) - Q_h^{\pi^k}(x_h^k, a_h^k)
= Q_h^k(B_h^k) - Q_h^{\pi^k}(x_h^k, a_x^k)\\
&= Q_h^k(B_h^k) - Q_h^\star(x_h^k, a_h^k) + Q_h^\star(x_h^k, a_h^k) - Q_h^{\pi^k}(x_h^k, a_x^k)\\
&\leq \rbr{1+\frac{1}{H}}\rbr{\one_{[t=0]}H + \sum_{i=1}^{t} \alpha_t^i(V_{h+1}^{k_i} - V^\star_{h+1})(x_{h+1}^{k_i})} + \clip\sbr{\beta_t \mid \frac{\gap_h(x_h^k, a_h^k)}{H+1}} \\
& \qquad+ (V^\star_{h+1} - V^{\pi^k}_{h+1})(x_{h+1}^k) + \xi_{h+1}^k.
\end{align*}
Summing over episodes, let $t_h^k=n_h^k(B_h^k)$ and let $k_i(B_h^k)$ be the episode where $t_h^k$ is incremented for the $i^{\textrm{th}}$ time.
\begin{align*}
\sum_{k=1}^K V_h^k(x_h^k) - V_h^{\pi^k}(x_h^k) &\le \sum_{k=1}^K \rbr{1+\frac{1}{H}}\left(\one_{[n_h^k(B_h^k)=0]}H + \clip\sbr{\beta_{n_h^k(B_h^k)}, \frac{\gap_h(x_h^k, a_h^k)}{H+1}}\right) \\
& \qquad + \rbr{1 + \frac{1}{H}}\sum_{k=1}^K\sum_{i=1}^{n_h^k(B_h^k)} \alpha_{n_h^k}^i(V_{h+1}^{k_i(B_h^k)} - V^\star_{h+1})(x_{h+1}^{k_i(B_h^k)})  \\
& \qquad+ \sum_{k=1}^K \left((V^\star_{h+1} - V^{\pi^k}_{h+1})(x_{h+1}^k) + \xi_{h+1}^k \right).
\end{align*}
For any ball $B$, let $T_0(B)$ be the first time that it is played but not updated, i.e., the first time that $\tilde{n}_h^k(B) > n_h^k(B)$. 
In the terminology of lemma~\ref{lem:lr_new}, for any ball $B$, we define the sequence $\tilde{\alpha}_t^i(T_0(B))$ with this value of $T_0$. 
Now, using the observation in~\citet{jin2018q,song2019efficient}, we can rearrange the second term and use lemma~\ref{lem:lr_new}:
\begin{align*}
\sum_{k=1}^K\sum_{i=1}^{n_h^k} \alpha^i_{n_h^k}(V_{h+1}^{k_i(B_h^k)} - V_{h+1}^\star)(x_{h+1}^{k_i(B_h^k)})
&\le \sum_{k=1}^{K} (V^k_{h+1} - V^\star_{h+1})(x_{h+1}^k) \sum_{t=n_h^k}^{\infty}  \tilde{\alpha}_t^{n_h^k}(T_0(B_h^k)) \\
&\le \rbr{1+\frac{2}{H}}^2 \sum_{k=1}^K (V_{h+1}^k - V^\star_{h+1})(x_{h+1}^k).
\end{align*}
The first inequality is based on the following reasoning: The left
hand side is ``backward'' looking, in the sense that for each episode
$k$ the expression involves the previous \emph{updates} to the ball
played. On the other hand, the right hand side is ``forward'' looking,
in that episode $k$ also results in an update to some ball (which may
not be the one that is played), and so it appears every subsequent
time the latter ball is \emph{played}. Thus, rather than looking at
the previous updates to the ball played in episode $k$, we can look at
the future \emph{plays} of the ball updated in episode $k$. This is
why we switch the weight sequence from $\alpha_{n_h^k}^i$ to
$\tilde{\alpha}_{t}^{n_h^k}$, where recall that the latter has every
$H^{\textrm{th}}$ term repeated, possibly after some initial burn-in
phase.

Since $V_{h+1}^{\pi^k}(x_{h+1}^k) \le V^\star_{h+1}(x_{h+1}^k)$, we have
\begin{align*}
&\rbr{1+\frac{1}{H}}\rbr{1+\frac{2}{H}}^2 \sum_{k=1}^K (V_{h+1}^k - V^\star_{h+1})(x_{h+1}^k) + \sum_{k=1}^K(V^\star_{h+1} - V_{h+1}^{\pi^k})(x_{h+1}^k) \\
&\le \rbr{1+\frac{2}{H}}^3 \rbr{\sum_{k=1}^K (V_{h+1}^k - V^\star_{h+1})(x_{h+1}^k) + \sum_{k=1}^K(V^\star_{h+1} - V_{h+1}^{\pi^k})(x_{h+1}^k)} \\
&=\rbr{1+\frac{2}{H}}^3 \sum_{k=1}^K(V^k_{h+1} - V_{h+1}^{\pi^k})(x_{h+1}^k).
\end{align*}
So we have
\begin{align*}
\sum_{k=1}^K (V_h^k - V_h^{\pi^k})(x_h^k) \le& \sum_{k=1}^K\rbr{1+\frac{1}{H}}\rbr{H\one_{[n_h^k(B_h^k)=0]} + \xi_{h+1}^k + \clip\sbr{\beta_{n_h^k(B_h^k)} \mid \frac{\gap_h(x_h^k, a_h^k)}{H+1}}} \\
&+ \rbr{1+\frac{2}{H}}^3 \sum_{k=1}^K(V^k_{h+1} - V_{h+1}^{\pi^k})(x_{h+1}^k).\tag*\qedhere
\end{align*}
\end{proof}

There are two terms that we need to bound. The $\xi_{h+1}^k$ term can
be bounded by a concentration argument as shown in
\cite{sinclair2019adaptive}.

\begin{lemma}[Azuma--Hoeffding bound, Lemma E.9 from \cite{sinclair2019adaptive}]
For any $\delta \in (0,1)$, with probability at least $1 - \delta / 2$
\begin{align*}
\sum_{h=1}^H \sum_{k=1}^k \xi_{h+1}^k \le 2\sqrt{2H^3K\log(4HK/\delta)}.
\end{align*}
\end{lemma}

The clipped $\beta_t$ term requires a more refined treatment to relate it to the zooming number or zooming dimension. Recall our definition of the near-optimal space
\begin{align*}
\Pcal_{h,r}^{Q^\star} = \{(x, a): \gap_h(x,a) \le c_1 r \},
\end{align*}
where $c_1 =\frac{2(H+1)}{d_{\max}}+ 2L$. Define the stage-dependent zooming number as
\begin{align*}
z_{h,c} = \inf\{ d>0 :  |\Pcal_{h,r}^{Q^\star}| \le cr^{-d}\}.
\end{align*}
The following is our key lemma that bounds surplus $\beta_t$ using the zooming number.

\begin{lemma}
\label{lem:beta_sum}
\begin{align*}
\sum_{h=1}^{H}\sum_{k=1}^K \clip[\beta_{n_h^k}, \frac{\gap_h(x_h^k, a_h^k)}{H+1}] \le &
\sum_{h=1}^{H} 32(\sqrt{H^3\log(4HK/\delta)} + L d_{\max}\sqrt{2H\Lambda}) \\
&\inf_{r_0 \in (0, d_{\max}]} \left( \sum_{r=d_{\max}2^{-i},r\ge r_0} N_r^{\text{pack}}(\Pcal_{h,r}^{Q^\star}) \frac{2d_{\max}\sqrt{2H\Lambda}}{r} + \frac{2Kr_0}{d_{\max}} \right).
\end{align*}
\end{lemma}

\begin{proof}
Let $c_2 = 16(\sqrt{H^3\log(4HK/\delta)} + L d_{\max}\sqrt{2H\Lambda})$. By Lemma \ref{lem:e7} we have
\begin{align*}
\beta_{n_h^k} \le 16(\sqrt{H^3\log(4HK/\delta)} + L d_{\max}\sqrt{H\Lambda+\Lambda+1}) \frac{1}{\sqrt{n_h^k}} \leq \frac{c_2}{\sqrt{n_h^k}}
\end{align*}
Let $N_{\mathrm{min}}(B) = \frac{1}{4} \left(\frac{d_{\max}}{r(B)}\right)^2$,
and $N_{\mathrm{max}}(B) = \left(\frac{d_{\max}}{r(B)}\right)^2(H\Lambda+\Lambda+1)$. Considering
Lemma~\ref{lem:numsample}, we know that whenever $\beta_{n_h^k}$ appears in our regret bound (which only happens when a ball is played), we have
\begin{align*}
N_{\mathrm{min}}(B) \le  n_h^k(B) \le N_{\mathrm{max}}(B).
\end{align*}
Letting $\gap_h(B) = \min_{(x,a)\in B} \gap_h(x,a)$ be the minimum gap $B$, we can
rearrange the sum for each ball.
\begin{align*}
\sum_{k=1}^K \clip\sbr{\beta_{n_h^k(B_h^k)}\mid \frac{\gap_h(x_h^k, a_h^k)}{H+1}} &\le \sum_{B\in \Pcal_h^K} \sum_{n=N_{\mathrm{min}}(B)}^{N_{\mathrm{max}}(B)} \clip\sbr{\frac{c_2}{\sqrt{n}}\mid \frac{\gap_h(B)}{H+1}}\\
&\le c_2\sum_{B\in \Pcal_h^K} \sum_{n=N_{\mathrm{min}}(B)}^{N_{\mathrm{max}}(B)} \clip\sbr{\frac{1}{\sqrt{n}}, \frac{\gap_h(B)}{H+1}}
\end{align*}
The last step is due to the fact that $c_2 > 1$ and if $ \frac{c_2}{\sqrt{n}} < \frac{\gap_h(B)}{H+1}$ then $\frac{1}{\sqrt{n}} < \frac{\gap_h(B)}{H+1}$.
Now, ignoring clipping, the inner sum can be bounded by
\begin{align*}
\sum_{n=N_{\mathrm{min}}(B)}^{N_{\mathrm{max}}(B)} \frac{1}{\sqrt{n}}
\le \int_{i=0}^{N_{\mathrm{max}}(B)}\frac{1}{\sqrt{i+ N_{\mathrm{min}}(B)}}
\leq 2\frac{d_{\max}\sqrt{H\Lambda+\Lambda+1}}{r(B)}.
\end{align*}

With clipping, we consider two cases.

\noindent \textbf{Case 1:} If $\gap_h(B) \ge \frac{2(H+1)r(B)}{d_{\max}}$, then the regret on ball $B$ will always be clipped: 
\begin{align*}
\frac{1}{\sqrt{n_h^k(B)}} \le \frac{1}{\sqrt{N_{\mathrm{min}}(B)}} = \frac{2 r(B)}{d_{\max}} \le \frac{\gap_h(B)}{H+1}.
\end{align*}

\noindent \textbf{Case 2:} If $\gap_h(B) < \frac{2(H+1)r(B)}{d_{\max}}$, then we
will pay $2d_{\max}\sqrt{2H\Lambda}/r(B)$ for this ball. However, we
will show that this ball also belongs to the near optimal set, so that
we do not incur this term too many times.

Let $(x_c, a_c)$ be the center of $B$ and $(x_m, a_m)\in B$ be the point that has the minimum gap, i.e. the point that achieves $\gap_h(B)$. Using the assumption that $Q^\star$ and $V^\star$ are Lipschitz:
\begin{align*}
\gap_h(x_c, a_c) - \gap_h(B) &= Q_h^\star(x_c, a_h^\star(x_c)) - Q^\star_h(x_c, a_c) - (Q_h^\star(x_m, a_h^\star(x_m)) - Q^\star_h(x_m, a_m)) \\
&\le 2Lr(B),
\end{align*}
so we know that all the points in $B$ have small gaps relative to $r$. In particular,
\begin{align*}
\gap_h(x_c,a_c) \le \gap_h(B) + 2Lr(B) \le \frac{2 (H+1)r(B)}{d_{\max}} + 2Lr(B).
\end{align*}
Thus, we have $(x_c, a_c) \in \Pcal_{h,r(B)}^{Q^\star}$.  Now we are
ready bound the sum. Note that for a ball $B \in \Pcal_h^K$, either
$B$ gets clipped, or the center of $B$ is in
$\Pcal_{h,r(B)}^{Q^\star}$.  Since all the balls of radius $r$ are at
least $r$ apart, we can have at most
$N_r^{\text{pack}}(\Pcal_{h,r}^{Q^\star})$ in the latter case.
\begin{align*}
\sum_{k=1}^K \clip\sbr{\beta_{n_h^k}\mid \frac{\gap_h(x_h^k, a_h^k)}{H+1}} &\le \sum_{B\in \Pcal_h^K} \sum_{n=N_{\mathrm{min}}(B)}^{N_{\mathrm{max}}(B)} \clip\sbr{\frac{c_2}{\sqrt{n}}\mid \frac{\gap_h(B)}{H+1}}\\
\hspace{-1ex}\le c_2 &\inf_{r_0 \in (0, d_{\max}]} \left( \sum_{r=d_{\max}2^{-i},r\ge r_0} N_r^{\text{pack}}(\Pcal_{h,r}^{Q^\star}) \frac{2  d_{\max}\sqrt{2H\Lambda}}{r} + \frac{2Kr_0}{d_{\max}} \right).
\end{align*}
The second term uses the fact that for any ball $B$ with $r(B) \leq r_0$, we have $N_{\mathrm{min}}(B) \leq \frac{1}{4}\rbr{\frac{d_{\max}}{r_0}}^2$.
\end{proof}

Now we are ready to prove Theorem \ref{thm:main}.
\begin{proof}[Proof of Theorem \ref{thm:main}]
We apply Lemma \ref{lem:cliprecur} recursively.
\begin{align*}
&\sum_{k=1}^K (V_1^k - V_1^{\pi^k})(x_1^k) \\
\le& (H+1) + \sum_{k=1}^K\rbr{1+\frac{1}{H}}\rbr{ \xi_{2}^k + \clip\sbr{\beta_{n_1^k(B_1^k)} \mid \frac{\gap_1(x_1^k, a_1^k)}{H+1}}} + \rbr{1+\frac{2}{H}}^3 \sum_{k=1}^K(V^k_{2} - V_{2}^{\pi^k})(x_{2}^k)\\
\le& \sum_{h=1}^{H} H\rbr{1+\frac{2}{H}}^{3(h-1)} + \sum_{h=1}^H \rbr{1+\frac{2}{H}}^{3h} \sum_{k=1}^K\rbr{\xi_{h+1}^k + \clip\sbr{\beta_{n_h^k(B_h^k)} \mid \frac{\gap_h(x_h^k, a_h^k)}{H+1}}} \\
\le& 404 H^2 + 404\sum_{h=1}^H\sum_{k=1}^K\rbr{ \clip\sbr{\beta_{n_h^k(B_h^k)} \mid \frac{\gap_h(x_h^k, a_h^k)}{H+1}}  + \xi_{h+1}^k}.
\end{align*}
Here we are using that $(1+2/H)^{3H} \leq \rbr{(1+2/H)^{H/2}}^{6}\leq e^6 \leq 404$. 
Next, we use the Azuma-Hoeffding inequality above to obtain:
\begin{align*}
  404\sum_{h=1}^H\sum_{k=1}^K \xi_{h+1}^k \leq 808\sqrt{2H^3K\log(4HK/\delta)}. 
\end{align*}
Finally, we use Lemma~\ref{lem:beta_sum} to bound the clipped surplus term:
\begin{align*}
404\sum_{h=1}^H\sum_{k=1}^K&\clip\sbr{\beta_{n_h^k(B_h^k)} \mid \frac{\gap(x_h^k,a_h^k)}{H+1}} \\
  & \leq 12928\sum_{h=1}^H\rbr{ \sqrt{H^3\log(4HK/\delta)} + Ld_{\max}\sqrt{2H\Lambda}} \\
  & ~~~~~~~~ \times \inf_{r_0 \in (0,d_{\max}]} \rbr{ \sum_{r=d_{\max}2^{-i}, r\geq r_0} N_r^{\text{pack}}(\Pcal_{h,r}^{Q^\star}) \frac{2d_{\max}\sqrt{2H\Lambda}}{r} + \frac{2Kr_0}{d_{\max}} }.
\end{align*}
Combining the above bounds, we obtain the theorem. 
\end{proof}

\end{document}